\documentclass{article}

\usepackage{arxiv}

\usepackage[utf8]{inputenc} 
\usepackage[T1]{fontenc}    
\usepackage{hyperref}       
\usepackage{url}            
\usepackage{booktabs}       
\usepackage{amsfonts}       
\usepackage{nicefrac}       
\usepackage{microtype}      
\usepackage{lipsum}
\usepackage{comment}
\usepackage{graphicx}

\usepackage{amsmath}
\usepackage{amsthm}
\usepackage{xcolor}

\newcommand{\PP}{\mathbb P} 
\newcommand{\R}{\mathbb R}

\newcommand{\rank}{\mbox{rank}}

\newcommand{\lra}{\longrightarrow}

\newcommand{\al}{\alpha}
\newcommand{\be}{\beta}
\newcommand{\ga}{\gamma}

\newcommand{\la}{\lambda}

\newtheorem{theorem}{Theorem}  
\newtheorem{corollary}[theorem]{Corollary}
\newtheorem{definition}[theorem]{Definition}
\newtheorem{lemma}[theorem]{Lemma}
\newtheorem{prop}[theorem]{Proposition}

\usepackage{amsbsy}

\title{General Deformations of Point Configurations Viewed By a Pinhole Model Camera}

\author{Yirmeyahu J. Kaminski\\
Holon Institute of Technology\\
Holon, Israel\\
{\tt\small jykaminski@gmail.com}
\And
Michael Werman\\
The Hebrew University of Jerusalem\\
Jerusalem, Israel\\
{\tt\small mchael.werman@mail.huji.ac.il}
}

\begin{document}
\maketitle

{\bf Keywords: Multiple-View Geometry, Deformation, Dynamic Scenes, Structure From Motion}

\begin{abstract}
This paper is a  theoretical study of the  Non-Rigid Structure from Motion problem: what can be computed from a monocular view of a parametrically deforming set of points? 
We treat various variations of this problem for 3D affine and general smooth deformations (under some mild technical restrictions) with either a calibrated or an uncalibrated camera.
We show that in general  at least three images related by quasi-identical  deformations 
are needed in order to have a finite set of solutions of the points' structure.
\end{abstract}

\section{Introduction}

Non-Rigid Structure from Motion (NRSfM) from a monocular camera has been addressed in a number of papers. The setup  is a  camera   tracking a deforming object. 
As the general problem is  unconstrained there have been many papers addressing certain specifications of the general case, for example  using a weak perspective or affine camera, \cite{Akhter_nonrigidstructure,DBLP:journals/pami/GotardoM11,Jingyu Yan and Marc Pollefeys,6247905,conf/eccv/AngstP12} or when the  deformation of the object is restricted to be  a linear combination of $k$ rigid shapes \cite{Hartley08perspectivenonrigid}. There are also a number of papers that constrain the deformation of the object to a physical model or a parameterized family of deformations which they then attempt to solve for  in an optimization framework \cite{6310582}. A review of much of the relevant literature can be found in \cite{Salzmann}.

The general case of deforming configurations of points has also received attention, but with some restrictions. Some authors consider configurations of points moving with constrained motions~\cite{Wolf-2001,Levin-2001}. Other papers  treat general motion but restrict their analysis to the case of a single point~\cite{Avidan-1999,Kaminski-2002,Kaminski-2004,Kaminski-2011}.  There has not been much published on the theoretical underpinnings of the recovery of structure of deforming configurations of points.


In this paper we analyze, for the first time, the complexity and ambiguities of a fixed perspective camera tracking a  parametrically set of  deforming body of points in 3D. When  the points move rigidly this is  the classic Structure from Motion (SfM) ~\cite{Hartley-Zisserman}.

The specific deformations we analyze are affine and more generally smooth deformations under mild restrictions. We show that when the camera is calibrated and the body undergoes  an affine deformation, a matching constraint similar to the classical epipolar geometry can be formulated. We show that from two images one cannot recover neither the deformation nor the original points. When three images (i.e. two deformations) are available, we show that in a generic situation, the remaining ambiguity is still three-dimensional. However, when the two deformations are quasi-identical (see below for a complete definition), there is exactly one solution.

We also show that an invariant shape description can be recovered from 3 images.
The recovery of this invariant does not require camera calibration.  

Then we turn our attention to the case of complete reconstruction (deformation and structure) for general smooth deformations. We show that if the deformation is slow with respect to the time frame and its spatial variations are small with respect to the mutual distances between the points, it can be calculated from a calibrated camera and 3 images, i.e. from the first view and a two other images coming from the same deformation repeated twice, like the affine distortion.

3D projective transformations are not treated as their images are indistinguishable from those of affine transformations;
$$
[I;0]
\begin{pmatrix}A_{3 \times 4} \\ a \,b \,c \,d \end{pmatrix} =
[I;0]\begin{pmatrix}A  \\ 0\, 0\, 0\,1\end{pmatrix}
$$

We are mainly interested in the theoretical possibilities both as to the number of corresponding points and images needed. We present complete algebraic solutions.

\section{Affine Deformations}

\begin{figure}
\includegraphics[width=0.60\textwidth]{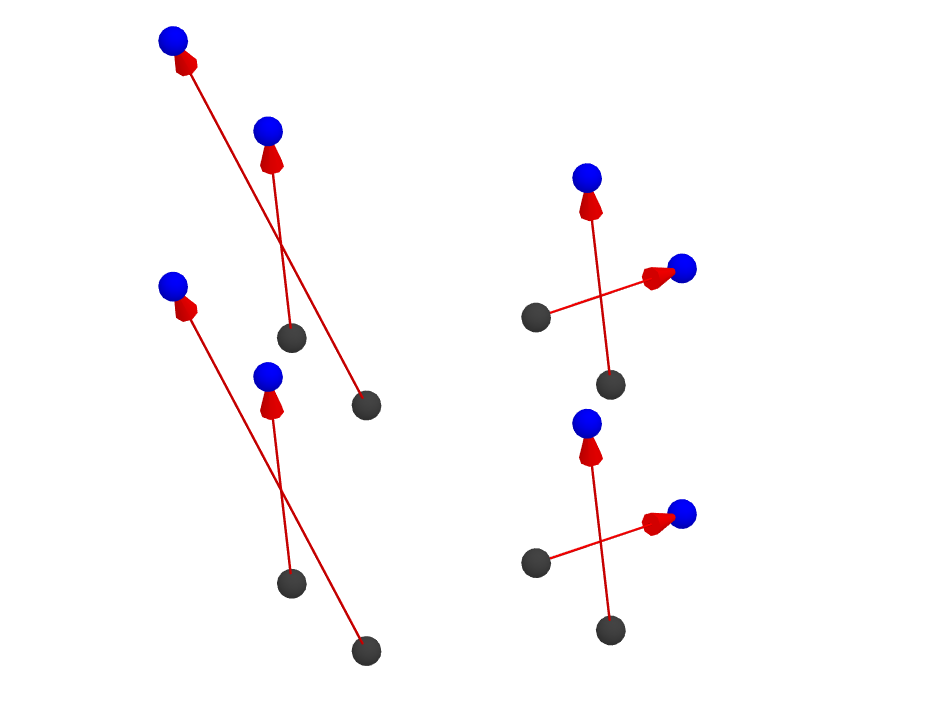}
\centering
 \caption{ The perspective
image of the vertices of a cube, in black, deforming affinely to the blue points, the arrows  are correspondences,  epipolar lines.}
\end{figure}
We start with the study of point correspondences under affine deformations. This has both a practical and a theoretical impact. On the practical side, we shall see that one can write a matching constraint  similar to  classical epipolar geometry. Thus, finding correspondences between two images of an affinely deforming configuration of points can  be done with the same machinery  as  in the classical case of  images of rigidly moving bodies. We consider only invertible affine deformations to avoid degenerate situations where  distinct points collapse to a single point after the deformation. 

On the theoretical front, we shall use this result extensively in the sequel. 

\subsection{Essential Matrix}
\label{sec::essentialMatrix}

Let us consider set of  deforming points  being imaged with a calibrated  camera, which without loss of generality  can be assumed to be $[I;0]$. 
The projection of a 3D point $P$ in homogeneous coordinates into  the first image is $q=[I;0]P$, while  the projection into the second image is  $q'=[I;0]\begin{pmatrix} A & t \\ 0 & 1\end{pmatrix} P$. Eliminating $P$ from these two sets of equations leads to a bilinear constraint over the  corresponding image points $q,q'$, the so-called {\em essential matrix $E$} where, $q'^t E q = 0$.  
 
\begin{lemma}
When $t \neq 0$, the essential matrix of this pair of images is: $E \equiv [t]_{\times} A$, where $[t]_\times$ is the matrix of the cross product with $t$ in the standard basis of $\R^3$. 
\end{lemma}
\begin{proof}
Consider a point $P$ in $\PP^3$, not at infinity, projected to  $q$ in the first image. Then $q \equiv [I;0]P$, thus $P = [\la q, 1]^t$, for some $\la \in \R$.  $[A;t]P$ is projected into the second image as $q' \equiv  [A;t]P$. Thus $q' \equiv (\la A q + t)$. Then $[t]_{\times} q' \equiv [t]_{\times} Aq$ as $[t]_{\times} t=0$. This yields $q'^t [t]_{\times} A q=0$ giving $E \equiv [t]_{\times} A$.
\end{proof}

If $t = 0$, $E$ cannot be computed because in that case, for any non-zero vector $u \in \R^3$, we have:  $q'^t [u]_{\times} A q = 0$.  

We denote by $\equiv$ equality modulo multiplication by a non-zero scalar. 

If there are matching pairs of points between the two images: $(q_i,q_i')_{i=1, \cdots,n}$, the following equations hold: \begin{equation} \label{eq::matching_pts} q_i'^t E q_i = 0 \end{equation} for each $i$ and rank $E$ is 2 as rank $[t]_{\times} $ is 2. Indeed the deformation is assumed to be invertible, i.e. $\det(A) \neq 0$. $E$ can thus  be computed from 7 pairs of corresponding points in general position or linearly from at least 8 pairs of points.

However, since $E$ has rank $2$ and is defined modulo multiplication by a non-zero scalar, the knowledge of it can only provide up to $7$ over the $12$ parameters that define an affine transformation. This implies that several images are necessary to compute the deformation. A more precise analysis is presented staring from the next section. 

\subsection{Deformation Recovery}
\label{sec::DefFromEssMat}

Once the essential matrix $E \equiv [t]_{\times}A$ is computed,  deformation can be recovered up to a 4-parameter ambiguity. To show this, we recall a lemma from \cite{Hartley-Zisserman} (page 255): If a rank $2$ matrix $F$ can be decomposed in two different ways as $F = [t]_{\times} A = [\tilde{t}]_{\times} \tilde{A}$ then there exists a constant
 $\la \neq 0$ and $v \in \R^3$, such that: $\tilde{t} = \la t$ and $\la \tilde{A} = A + tv^t$. Notice that since the matrix $E$ is defined modulo $\R^*$ (multiplication by a non-zero scalar), there are 5 and not 4 degrees of freedom related to the extracted of the deformation from $E$. 
 

Can more than two images help? Let us consider the situation where the deformation between the first and the second image is $(A,a)$ and  between the second and the third   $(B,b)$. We  now have three distinct essential matrices: $E_{12} \equiv [a]_{\times} A$, $E_{23} \equiv [b]_{\times}B$ and $E_{13} \equiv [Ba+b]_{\times}BA$. 

From $E_{12}$, we can compute $a_0$ and $A_0$ such that $\exists \alpha \neq 0, a = \alpha a_0$ and $\exists v_1, A = \frac{1}{\alpha}(A_0 + a_0 v_1^t)$. From $E_{23}$, we can compute $b_0$ and $B_0$ such that $\exists \beta \neq 0, b = \beta b_0$ and $\exists v_2, B = \frac{1}{\beta}(B_0 + b_0 v_2^t)$. From the third essential matrix $E_{13}$ we can compute $c_0$ and $C_0$ such that:
\begin{equation}
\label{eq::gamma}
  \exists \gamma \neq 0, Ba+b = \gamma c_0
\end{equation} 
and
\begin{equation}
\label{eq::v_3}
  \exists v_3, BA = \frac{1}{\gamma}(C_0 + c_0v_3^t)  
\end{equation}

From equations~(\ref{eq::gamma}) and~(\ref{eq::v_3}), we get the following system: 
\begin{equation}
\label{eq::affine_system}
\left \{ \begin{array}{lcc}
\alpha (B_0 + b_0 v_2^t) a_0 + \beta^2 b_0 - \beta \gamma c_0 & = & 0\\
\gamma (B_0 + b_0 v_2^t) (A_0 + a_0 v_1^t) - \alpha \beta (C_0 +c_0v_3^t) & = & 0
\end{array} \right. 
\end{equation}

Furthermore one has to enforce the  constraint that none of $\alpha,\beta$ nor $\gamma$ vanishes. Formally, this is equivalent to computing  in the localization of the polynomial ring with respect to these variables~\cite{GreuelPfister-2002}. Concretely, one must introduce new variables: $x,y,z$ and the equations: $\alpha x - 1 = \beta y - 1 = \gamma z -1 = 0$.

Eventually the number of unknowns is $N = 9+3+3 = 15$ and we have exactly $15$ equations. They define an real algebraic variety $X$ of $\R^{15}$. By~\cite{Mather-2012}, it is  known that real algebraic varieties are  stratified manifolds. Roughly speaking a stratified manifold has a dense open set which is a smooth manifold and whose complement is a stratified manifold of strictly smaller dimension. Further technical conditions are needed to fully define a stratified manifold. These conditions are satisfied in the case of real algebraic varieties. See~\cite{Mather-2012} for more details.  

Our concern now is to determine the dimension of $X$. $X$ is a finite set only if the dimension is zero. In this case,   its degree is useful in order to estimate the number of solutions. We prove here that $X$ has strictly positive dimension.


\begin{theorem}
\label{theo::3IsBad}
  For two unrelated non-singular deformations, such that $Ba$ and $b$ are linearly independent, $X$ is a three-dimensional manifold diffeomorphic to $\{-1,1\} \times \R^3$.    
\end{theorem}
\begin{proof}
Let $\alpha_0, \beta_0, \gamma_0, x_0, y_0, z_0, v_{10}, v_{20}, v_{30}$ be a point on $X$. $X$ cannot be empty since at least the actual deformations must satisfy the equations defining $X$. Let $a = \alpha_0 a_0$, $A = 1/\alpha_0(A_0 + a_0 v_{10}^t)$, $b = \beta_0 a_0$, $B = 1/\beta_0(B_0 + b_0 v_{20}^t)$. Then we know that $Ba+b = \gamma_0 c_0$ and $BA = 1/\gamma_0(C_0 + c_0 v_{30}^t)$. Consider the variety $Y$ defined by the following system:
$$
\left \{ \begin{array}{lcc}
\alpha' (B + b v_2'^t) a + \beta'^2 b - \beta' \gamma' (Ba+b) & = & 0\\
\gamma' (B + b v_2'^t) (A + a v_1'^t) - \alpha' \beta' (BA + (Ba+b)v_3'^t) & = & 0 \\
\alpha' x' - 1 & = & 0 \\
\beta' y' - 1 & = & 0 \\
\gamma' z' - 1 & = & 0
\end{array} \right. 
$$ 

The varieties $X$ and $Y$ are easily seen to be isomorphic by the linear mapping: 
$$
\begin{array}{l}
(\alpha,\beta,\gamma,x,y,z,v_1^t,v_2^t,v_3^t) \mapsto \\ 
\hspace{1cm} \begin{array}{l} 
(\alpha',\beta',\gamma',x',y',z',v_1'^t,v_2'^t,v_3'^t) = \\
(\alpha/\alpha_0 , \beta/\beta_0 , \gamma/\gamma_0, \alpha_0 x, \beta_0 y, \gamma_0 z, 1/\alpha_0^2 (v_1^t - v_{10}^t) , 1/\beta_0^2 (v_2^t - v_{20}^t) , 1/\gamma_0^2 (v_3^t - v_{30}^t)) \end{array} 
\end{array}
$$ 

For the sake of clarity and simplicity, we shall drop the prime in all variables. For instance, we shall continue to write $\alpha$ while we intend $\alpha'$ and similarly for all variables. 

Assume that $(A,a,B,b)$ is given and satisfies the assumptions of the theorem. The first equation yields $(\alpha - \beta \gamma)Ba + (\alpha(v_2^ta) + \beta^2 - \beta \gamma)b = 0$. Since the vectors $Ba,b$ are linearly independent, we have: 
\begin{equation}
\label{eq::abc}
  \alpha - \beta \gamma = 0
\end{equation} and 
\begin{equation}
\label{eq::v2}
  \alpha(v_2^ta) + \beta^2 - \beta \gamma=0  
\end{equation}

The second equation yields $(\gamma - \alpha \beta)BA + Ba(\gamma v_1^t - \alpha \beta v_3^t) + b(\gamma v_2^tA + \gamma(v_2^ta)v_1^t - \alpha \beta v_3^t) = 0$. Since $\rank(BA) = 3$ and $\rank(Ba(\gamma v_1^t - \alpha \beta v_3^t) + b(\gamma v_2^tA + \gamma(v_2^ta)v_1^t - \alpha \beta v_3^t)) \leq 2$, we have 
\begin{equation}
\label{eq::abc2}
\gamma-\alpha \beta = 0  
\end{equation} and
 
\begin{equation}
\label{eq::v3}
Ba(\gamma v_1^t - \alpha \beta v_3^t) + b(\gamma v_2^tA + \gamma(v_2^ta)v_1^t - \alpha \beta v_3^t) = 0
\end{equation}

Relying on equations~\ref{eq::abc} and~\ref{eq::abc2}, we get that $\beta \in \{-1,1\}$. 

If $\beta = 1$, we get $\alpha = \gamma$ and so by equation~\ref{eq::v3}, we have $Ba(\alpha v_1^t - \alpha v_3^t) + b(\alpha v_2^tA + \alpha(v_2^ta)v_1^t - \alpha v_3^t) = 0$. Again, since $Ba$ and $b$ are linearly independent, we get $v_1 = v_3$ and $v_2^t A + (v_2^ta) v_1^t - v_1^t=0$. Equation~\ref{eq::v2} yields $v_2^ta = \frac{\alpha-1}{\alpha}$. Hence $v_2 = \frac{1}{\alpha} A^{-t}v_1$. Then by $v_2^ta = \frac{\alpha-1}{\alpha}$, we finally get $\alpha = 1 +v_1^t A^{-1} a$. Proving that given $v_1$, one can compute linearly $v_2,v_3,\alpha,\gamma$. Therefore the connected component of $Y$ on which $\beta = 1$ is indeed a manifold diffeomorphic to $\R^3$. 

Now if $\beta = -1$, the same technique yields a similar conclusion. Indeed here $\gamma = -\alpha$, which leads to $Ba(-\alpha v_1^t - \alpha v_3^t) + b(-\alpha v_2^tA - \alpha(v_2^ta)v_1^t - \alpha v_3^t) = 0$. The linear independence of $Ba$ and $b$ provides us with the following constraints: $v_3 = -v_1$ and $v_2^t A + (v_2^ta) v_1^t - v_1^t=0$. The latter one is identical to the constraint in the  $\beta = 1$ case. Therefore we have a similar conclusion and the connected component of $Y$ for which $\beta = -1$ is a manifold diffeomorphic to $\R^3$ too.

From these two cases, the conclusion of the theorem follows. 
\end{proof}

In relation to the discussion above, one can check that in the neighborhood of each real solution, there are infinitely other real solutions. For example, provided that $(A,a),(B,b)$ is a solution,  $(\la A,a),(B,b)$ for $\la \in \R \backslash \{0\}$ is also a solution, since the essential matrices remains unchanged up to a scale.

The practical consequence of this theorem is that one cannot hope to recover deformations from three images in the general case.

When the same deformation is repeated twice the system of equations is simplified. 

Before we proceed more in depth, let us make the following observation. The deformations $(A,a)$ and $(\lambda A, \lambda a)$ for $\lambda \neq 0$ produce the same image. Therefore one could conclude  that whatever the number of images, one can only expect  to recover the deformation modulo this equivalence. However observe that if multiples $(\lambda A, \lambda a)$ and $(\mu A, \mu a)$ of the same deformation are applied consecutively we get the following overall deformation $(\mu \lambda A^2, \mu \lambda Aa + \mu a)$ which is equivalent to $(A^2, Aa + a)$ only if $\mu \lambda = \mu$ or equivalently $\lambda = 1$. Therefore if exactly the same deformation is repeated twice, one can hope to be able to fully recover it. This is the conclusion the  analysis below exhibits.

Consider the essential matrices $E_{12}$ and $E_{13}$. $E_{23}$ is the same  as $E_{12}$, because the same deformation is repeated twice. We compute $A_0,a_0,C_0,c_0$ as previously and for the actual deformation $A,a$ there exist $\alpha \neq 0$, $v_1 \in \R^3$, $\gamma \neq 0$ and $v_3 \in \R^3$, such that:
$$
\left \{ \begin{array}{l}
a = \alpha a_0 \\
A = 1/\alpha (A_0 + a_0 v_1^t) \\ 
Aa + a = \gamma c_0 \\
A^2 = 1/\gamma (C_0 + c_0 v_3^t)
\end{array} \right .
$$ 

This results in the following system of equations:
\begin{equation}
\label{eq::2IsBetter}
\left \{ \begin{array}{lcc}
(A_0 + a_0 v_1^t) a_0 + \alpha a_0 - \gamma c_0 & = & 0\\
\gamma (A_0 + a_0 v_1^t)^2  - \alpha^2 (C_0 + c_0 v_3^t) & = & 0 \\
\alpha x - 1 & = & 0 \\
\gamma z - 1 & = & 0
\end{array} \right .
\end{equation}

\begin{theorem}
\label{theo::2IsBetter}
For a generic affine deformation, such that the three following conditions hold (i) $\rank(A) = 3$, (ii) $1$ is not an eigenvalue of $A$ and (iii) $Aa,a$ are linearly independent, repeated twice, one can recover this deformation from the three images.   
\end{theorem}
\begin{proof}
 Here $X$ designates the sub-variety of $\R^{10}$ defined by the system~\eqref{eq::2IsBetter}. 
 Let $\alpha_0, \beta_0, x_0, z_0, v_{10}, v_{30}$ be a point on $X$. Let $a = \alpha_0 a_0$, $A = 1/\alpha_0(A_0 + a_0 v_{10}^t)$. Then we know that $Aa+a = \gamma_0 c_0$ and $A^2 = 1/\gamma_0(C_0 + c_0 v_{30}^t)$. Consider the variety $Y$ defined by the following system:
\begin{equation}
\label{eq::2IsBetter_modified}
\left \{ \begin{array}{lcc}
(A + a v_1'^t) a  + \alpha' a - \gamma' (Aa+a) & = & 0\\
\gamma' (A + a v_1'^t)^2 - \alpha'^2(A^2 + (Aa+a)v_3'^t) & = & 0 \\
\alpha' x' - 1 & = & 0 \\
\gamma' z' - 1 & = & 0
\end{array} \right . 
\end{equation}

The varieties $X$ and $Y$ are easily seen to be isomorphic. Indeed the following linear mapping: $(\alpha,\gamma,x,y,z,v_1^t,v_3^t) \mapsto (\alpha',\gamma',x',z',v_1'^t,v_3'^t) = (\alpha/\alpha_0 , \gamma/\gamma_0, \alpha_0 x, \beta_0 y, \gamma_0 z, 1/\alpha_0^2 (v_1^t - v_{10}^t) , 1/\gamma_0^2 (v_3^t - v_{30}^t))$ is an isomorphism from $X$ and $Y$. Therefore $\dim(X) = \dim(Y)$. 

As before, we shall drop the prime from all variables in order to ease the expressions.

The first equation yields 
\begin{equation}
(1-\gamma)Aa + ((v_1^ta) + \alpha - \gamma)a = 0.   
\end{equation}
Since $Aa$ and $a$ are linearly independent, we get $\gamma=1$ and $v_1^t a = 1-\alpha$. The second equation yields $(1-\alpha^2)A^2 + Aa(v_1^t - \alpha^2 v_3^t) + a(v_1^t A + (v_1^ta) v_1^t - \alpha^2 v_3^t) = 0$. Since $\rank(A^2) = 3$ and $\rank(Aa(v_1^t - \alpha^2 v_3^t) + a(v_1^t A + (v_1^ta) v_1^t - \alpha^2 v_3^t)) \leq 2$, this yields $1-\alpha^2 = 0$ and $v_1^t - \alpha^2 v_3^t = v_1^t A + (v_1^ta) v_1^t - \alpha^2 v_3^t = 0$ (because $Aa$ and $a$ are linearly independent). Hence $v_1 = \alpha^2 v_3$ and $v_1^t A + (v_1^ta) v_1^t - v_1^t= 0$. Since $v_1^t a = 1-\alpha$, we get $A^t v_1 = \alpha v_1$. Since $1 \not \in spec(A^t) = spec(A)$, $\alpha \neq 1$ and then $\alpha = -1$. Then $v_1$ is an eigenvector of $A^t$ with respect to $-1$. Together with $v_1^t a = 1-\alpha = 2$, one can compute $v_1$ and then there is a unique solution to the system, since the other variables can be computed from $\alpha$ and $v_1$. 
\end{proof}
 
It is clear that the unique solution is real, since this is  the actual deformation that the points have undergone. 

There are cases, other than two identical transformations, where the deformations are also solvable. 

For example, when $B = \la A$ and $b = \mu a$ for unknown, non-zero scalars  $\la, \mu$. The  system of equations~\ref{eq::affine_system}  reduces to:

\begin{equation}
\label{eq::quasi_identical_1}
\left \{ \begin{array}{lcc}
\la (A_0 + a_0 v_1^t) a_0 + \alpha \mu  a_0 - \gamma c_0 & = & 0\\
\la \gamma (A_0 + a_0 v_1^t)^2 - \alpha^2 (C_0 +c_0v_3^t) & = & 0
\end{array} \right. 
\end{equation}

This system is  similar, but still different, than system~\eqref{eq::2IsBetter}. Of course, as previously, one has to add the two further equations $\alpha x - 1 = \gamma z - 1 = 0$. Now we shall prove the following result.

System~\eqref{eq::quasi_identical_1} defines a discrete variety. As a consequence, 
\begin{theorem}
\label{thm::2images_is_better++}
 If $(A,a)$ is the first deformation and $(\la A, \mu a)$ 
 the second deformation ($\la \neq 0$ and $\mu \neq 0$), one can recover the two deformations and the structure provided that $Aa,a$ are linearly independent and $\frac{\mu}{\lambda} \not \in spec(A)$ is known. 
\end{theorem}
\begin{proof}
 We proceed as in the the previous theorem. Here $X$ designates the sub-variety of $\R^{10}$ defined by  system~\eqref{eq::quasi_identical_1} (together with equations $\alpha x - 1 = \gamma z - 1 = 0 $). Note that $\lambda , \mu$ are not unknowns but parameters. 
 Let $\alpha_0, \gamma_0, x_0, z_0, v_{10}, v_{30}$ be a point on $X$. Let $a = \alpha_0 a_0$, $A = 1/\alpha_0(A_0 + a_0 v_{10}^t)$. Then we know that $\la Aa + \mu a = \gamma_0 c_0$ and $\la A^2 = 1/\gamma_0(C_0 + c_0 v_{30}^t)$. Consider the variety $Y$ defined by the following system:
\begin{equation}
\label{eq::quasi_identical_2}
\left \{ \begin{array}{lcc}
\la (A + a v_1'^t) a  + \alpha' \mu a - \gamma' (\la Aa+ \mu a) & = & 0\\
\la \gamma' (A + a v_1'^t)^2 - \alpha'^2(\la A^2 + (\la Aa + \mu a)v_3'^t) & = & 0 \\
\alpha' x' - 1 & = & 0 \\
\gamma' z' - 1 & = & 0 \\
\end{array} \right . 
\end{equation}

Again the two varieties $X$ and $Y$ are easily seen to be isomorphic. And as before, we shall drop the prime from all variables.

Form the second equation, we get $\lambda(1-\gamma)Aa + (\lambda v_1^t a + \alpha \mu - \gamma \mu) a = 0$. Therefore $\gamma = 1$ and $v_1^t a = \frac{\mu}{\lambda}(1-\alpha)$. 

From the second equation we get: $\lambda(1-\alpha^2)A^2 + \lambda Aav_1^t + \lambda a v_1^t A + \lambda (v_1^t a) a v_1^t - \alpha^2 \lambda Aa v_3^t - \alpha^2 \mu a v_3^t = 0$. 

Relying on a rank argument, as above, we get $1-\alpha^2 = 0$ and $\lambda Aa (v_1^t - \alpha^2 v_3^t) + a (\lambda v_1^t A + \lambda (v_1^t a) v_1^t - \alpha^2 \mu v_3^t) = 0$.  

The linear independence of $Aa$ and $a$ again implies that $v_1^t - \alpha^2 v_3^t = 0$ and $\lambda v_1^t A + \lambda (v_1^t a) v_1^t - \alpha^2 \mu v_3^t = 0$. This yields $v_3 = \frac{1}{\alpha^2}v_1$ and $A^t v_1 = \frac{\mu \alpha}{\lambda} v_1$. If $\alpha = 1$, then $v_1$ would be an eigenvector of $A^t$ with respect to $\frac{\mu}{\lambda}$, which contradicts the assumption. Then $\alpha = -1$ and then one can compute $v_1$ relying on $A^t v_1 = \frac{\mu \alpha}{\lambda} v_1$ and $v_1^t a = \frac{\mu}{\lambda}(1-\alpha)$. From this, one gets $v_3$. Thus there is a unique solution.  
\end{proof}

On the practical side, since the ratio $\frac{\mu}{\lambda}$ must be known for the computation to be carried out, one can assume that $\mu = \lambda$. This situation will formalized in definition~\ref{def::homothety_equivalence} below. 

\subsection{Beyond Essential Matrices}
\label{sec::recons}

Consider now the computation of both the deformation and the structure without computing the essential matrix. Let $P_1, \cdots, P_n$ be $n$ points in $\R^3$ that undergo an affine deformation $\left ( \begin{matrix} A & t \\ 0 & 1 \end{matrix} \right )$. Before the deformation the image points are $q_i = (u_i,v_i,1)^t$ and after the deformation are denoted $q_i' = (u_i',v_i',1)^t$. The camera matrix is still $[I,0]$

With these notations, there exists for each $i$, $\lambda_i \in \R \backslash \{0\}$, such that $P_i = \lambda _i q_i$. Hence we have the following set of equations:

\begin{equation}
\label{eq::full_eq}
q_i'  \equiv  \lambda_i A q_i + t 
\end{equation}

First notice that this equation is an equality in the projective plane. From a set of such equation, one cannot expect to fully compute the deformation. 

However we shall show that one can compute the deformation modulo an overall scale and fully recover the structure, provided $t \neq 0$ and $4$ points are known in $\R^3$.  

If $t = 0$, equality \eqref{eq::full_eq} reads $q_i' \equiv A q_i$ and one cannot recover the  structure (i.e. $\lambda_i$) at all. Therefore, we assume in the sequel that $t \neq 0$. 

\begin{definition}
\label{def::homothety_equivalence}
Two affine deformations $(A,a)$ and $(B,b)$ are said to be homothety equivalent if there exists a non-zero real $\lambda$ such that $B = \lambda A$ and $b = \lambda a$. 
\end{definition}

\begin{lemma}
Assume $n$ image correspondences $q_i \leftrightarrow  q_i'$, before and after deformation, are given.
Assume that the initial structure is known, that is $\lambda_1, \cdots, \lambda_n$ are known. Then the set of deformations that can be computed in this setting is a one-dimensional linear space, provided that $n \geq 4$ and the points $\{P_i = \lambda_i q_i\}_{1 \leq i \leq n}$ are in a generic position. 
\end{lemma}
\begin{proof}
Equation~\eqref{eq::full_eq} is says that $AP_i + t$ lies in the ray defined by the camera center and the image point $q_i'$. 

Let $\phi_1(P) = AP+a$ and $\phi_2(P) = BP +b$ be two invertible affine deformations that are compatible with equations~\eqref{eq::full_eq} for $i=1, \cdots, n$. Then there is an affine transformation $h$ that maps $\phi_1(P_i)$ to $\phi_2(P_i)$ for each $i$, say $h = \phi_2 \circ \phi_1^{-1}$. More precisely $h(Q) = BA^{-1}Q - BA^{-1}a + b$. Let $Q_i = \phi_1(P_i)$. For each $i$, the origin, $Q_i$ and $h(Q_i)$ are aligned. Then provided that $n \geq 4$ and points are in a generic configuration, $h$ is an homothety, that is $b = BA^{-1}a$ and $BA^{-1} = \sigma I$ for some $\sigma \neq 0$. Indeed let $v_i$ be the vector $\overrightarrow{OQ_i}$, such that $v_1,v_2,v_3$ form a basis of $\R^3$. Then we have $h(v_i) = \sigma_i v_i$ for $1 \leq i \leq 3$ and some non-zero scalars $\sigma_1,\sigma_2,\sigma_3$. Now let $v_4 = \overrightarrow{OQ_4}$, so that $v_4$ is a linear combination of $v_1,v_2,v_3$: $v_4 = \alpha_1 v_1 + \alpha_2 v_2 + \alpha_3 v_3$. Then $h(v_4) = \sigma_4 v_4 = \alpha_1 \sigma_1 v_1 + \alpha_2 \sigma_2 v_2 + \alpha_3 \sigma_3 v_3$, so that $\sigma_1 = \sigma_2 = \sigma_3 = \sigma_4$ and $h$ is an homothety as expected. 

Hence $B = \sigma A$ and $b = \sigma a$. Therefore under the assumptions of the lemma, two deformations that are compatible with the equations~\eqref{eq::full_eq} are homothety equivalent.                                              

Now consider the three first points that define a non-degenerate triangle. For the sake of simplicity, denote them  $P_1, P_2, P_3$. The plane defined by these point, say $H$, intersects the rays defined by $q_1',q_2',q_3'$ in three points $Q_1,Q_2,Q_3$. The correspondences $P_i \mapsto Q_i$ provide $9$ linear independent constraints on the deformation $(A,a)$. 

Now consider a fourth point $P_4$ not lying on $H$. Saying that $AP_4+a$ lies in the rays defined by $q_4'$ adds two linear constraint independent of the previous ones. 

We ends up with 11 linear independent constraints. This allows computing an element in the equivalent class of the actual deformation, say $(A_0,a_0)$. As mentioned previously the two deformations $(A,a)$ and $(A_0,a_0)$ are homothety equivalent.  

The group of homotheties centered at the origin is a one-dimensional linear space, which yields the conclusion. 
\end{proof}

Now equation~\eqref{eq::full_eq} implies 

\begin{equation}
\label{eq::eqs_on_lambda_i}
    \left\{ \begin{array}{rcl}
         \lambda_i a_1^t q_i + t_1 - u_i' (\lambda_i a_3^t q_i + t_3) & = & 0  \\
         \lambda_i a_2^t q_i + t_2 - v_i' (\lambda_i a_3^t q_i + t_3) & = & 0 
    \end{array}\right.,
\end{equation}
 where we  denote by $a_1^t,a_2^t,a_3^t$ the line of $A$ and $t_1,t_2,t_3$ the coordinates of $t$. Since the points are assumed to be in $\R^3$ and therefore do not lie at infinity, the coefficients $(\lambda_i a_3^t q_i + t_3)$ do not vanish, since affine transformations do not send points to infinity. Therefore the implication is actually an equivalence.

For each $i$, let us consider the following function $f_i: \R^{12} \times \R \longrightarrow \R^2$ that maps $(a_{11}, \cdots, a_{33}, t_1, t_2, t_3, \lambda_i)$ to $(\lambda_i a_1^t q_i + t_1 - u_i' (\lambda_i a_3^t q_i + t_3), \lambda_i a_2^t q_i + t_2 - v_i' (\lambda_i a_3^t q_i + t_3))$. The Jacobian matrix of $f_i$ is:

\begin{equation}
\label{eq::jac_matrix_i}    
\left [\begin{array}{ccccccccccccc}
\lambda_i u_i & \lambda_i v_i & \lambda_i & 0 & 0 & 0 & - \lambda_i u_i u_i' & - \lambda_i v_i u_i' & - \lambda_i u_i'& 1 & 0 & -u_i' & a_1^tq_i - u_i' a_3^t q_i \\
 0 & 0 & 0 & \lambda_i u_i & \lambda_i v_i & \lambda_i & - \lambda_i u_i v_i' & - \lambda_i v_i v_i' & - \lambda_i v_i' & 0 & 1 & -v_i' & a_2^tq_i - v_i' a_3^t q_i 
\end{array} \right ]
\end{equation}

This matrix has always rank $2$ and $f_i$ is therefore a submersion from $\R^{12} \times \R$ to $\R^2$. The level set over $0$ is not empty since the actual deformation and structure define a point in it. Therefore the $f_i^{-1}(0)$ is actually a smooth manifold of dimension $13-2=11$. 

Now let us consider for each $k$ in $\{1, \cdots, n\}$, the injection $i_k^n: \R^{12} \times \R \longrightarrow \R^{12} \times \R^n$, such that $i_k^n(a_{11}, \cdots, a_{33},t_1,t_2,t_3,\lambda_k) = (a_{11}, \cdots, a_{33},t_1,t_2,t_3,0, \cdots, 0, \lambda_k, 0, \cdots, 0)$, where $\lambda_k$ is sent to the position $k$ in the second factor in the product $\R^{12} \times \R^n$.  Let $\pi_k^n$ be the left inverse of $i_k^n$, that is the projection from $\R^{12} \times \R^n$ to $\R^{12} \times \R$, where the $k-$copy of $\R$ is the only factor that is kept. 

\begin{prop}
For $n \in \{1,\cdots,7\}$, the set $\cap_{i=1}^n (f_i \circ \pi_i^n)^{-1}(0)$ is a smooth manifold of dimension $12-n$. 
\end{prop}
\begin{proof}
The proposition holds for $n=1$ according to the above analysis. Assume that it is true for some $n \leq 6$, let us prove it for $n+1$. 
By the induction assumption we have $\dim(\cap_{i=1}^n (f_i \circ \pi_i^n)^{-1}(0)) = 12-n$. 

Since we now add a point, a direction is added and now we shall look at $M = \cap_{i=1}^{n} (f_i \circ \pi_i^{n+1})^{-1}(0)$ in place of $\cap_{i=1}^{n} (f_i \circ \pi_i^n)^{-1}(0))$. Therefore $\dim(M) = 12-n+1$. 

Now let $N$ be $(f_{n+1} \circ \pi_{n+1}^{n+1})^{-1} (0)$. By the above analysis, $N$ is a smooth manifold of dimension $12+n+1 - 2 = 12+n-1$. The two manifolds $M$ and $N$ are transverse, since for each $x \in M \cap N$, we have $T_x M + T_x N = T_x (\R^{12} \times \R^{n+1})$. Indeed, if $\{e_i\}_{1 \leq i \leq 12+n+1}$ denotes the standard basis of $\R^{12+n+1}$ then for $x = (a_{11}, \cdots, a_{33}, t_1, t_2, t_3, \lambda_1, \cdots, \lambda_{n+1})$, the vectors $z_1 = e_{12+n+1}$ and $z_2 = -\prod_{i=1}^n (a_1^t q_{i} - u'_{i}a_3^t q_{i}) e_{10} + \sum_{k=1}^n \prod_{i \neq k} (a_1^t q_{i} - u'_{i}a_3^t q_{i}) e_{12+k}$ are linearly independent, lie in $T_x M$, but not in $T_x N$. 

Therefore $T_x N + \R z_1 + \R z_2 = T_x (\R^{12} \times \R^{n+1}) \subset T_x N + T_x M$. 

Then $M \cap N = \cap_{i=1}^{n+1} (f_i \circ \pi_i^{n+1})^{-1}(0)$ is a manifold, which dimension is $12-n+1 + 12+n-1 - (12 + n + 1) = 12 - n -1$ as expected. This completes the induction.  
\end{proof}

Adding further points does not decrease the dimension. More precisely, we now prove the following result.

\begin{theorem}
\label{thm::essential_matrix_is_all}
Given $n \geq 7$ correspondences $q_i \leftrightarrow  q_i'$ of points that are in a generic configuration (see below section~\ref{sec::critical_surface} for more details), the set $\cap_{i=1}^n (f_i \circ \pi_i^n)^{-1}(0)$ is a smooth submanifold of $\R^{12+n}$ of dimension $5$.  
\end{theorem}
\begin{proof}
For the sake of simplicity, let us denote $\cap_{i=1}^n (f_i \circ \pi_i^n)^{-1}(0)$ by M. Consider the canonical projection $\pi: \R^{12+n} \rightarrow \R^{12}$ on the 12 first coordinates. The image of M by $\pi$ is obtained by eliminating $\lambda_i$ for each $i$ from the two equations each correspondence provides. This accounts to say that $\pi(M)$ is defined by the following equations:
$$
(a_1^t q_i - u_i' a_3^t q_i)(v_i' t_3 - t_2) = (a_2^t q_i - v_i' a_3^t q_i)(u_i' t_3 - t_1),
$$
for $i$ in $\{1, \cdots, n\}$. After simplification, this yields:
$$
v_i' t_3 a_1^t q_i - t_2 a_1^t q_i + t_2 u_i' a_3^t q_i = u_i' t_3 a_2^t q_i - t_1 a_2^t q_i + t_1 v_i' a_3^t q_i, 
$$
which is a bi-linear relation on $q_i$ and $q_i'$. This relation can also be written in a matrix form:
$$
q_i'^t [t]_{\times} A q_i = 0.
$$
So with no surprise, we roll back to the essential matrix, which was obtained by eliminating the 3d points, while here we  eliminated $\lambda_i$, which is equivalent. From the beginning of the section~\ref{sec::essentialMatrix} it appears clearly that, provided the points are in a generic configuration which allows the computation of $E$, the set $\pi(M)$ is indeed a 5 dimensional smooth manifold parametrized by $(\R^*)^2 \times \R^3$ and embedded into $\R^{12}$. 

Now let us consider the fiber of $\pi$ over a point $z = (a_1^1,a_2^t,a_3^t,t_1,t_2,t_3) \in \pi(M)$. For each $i$, one has two consistent equations on $\lambda_i$ (Equations~\eqref{eq::eqs_on_lambda_i}). Therefore, the fiber is a discrete set parametrized by the $n$ values $\lambda_1, \cdots, \lambda_n$. 

The fiber being always a zero-dimensional manifold, diffeomorphic to $\{1, \cdots, n\}$, such that for each point $z \in \pi(M)$, there is an open neighborhood $U_z \subset \pi(M)$ and a bijection $\phi_z: \pi^{-1}(U_z) \rightarrow U_z \times \{1, \cdots, n\}$ that satisfies: (i) $\text{pr}_1 \circ \phi_z = \pi_{\mid \pi^{-1}(U_z)}$ (where $\text{pr}_1: U_z \times \{1, \cdots, n\} \rightarrow U_z$ is the projection on the first factor) and (ii) for each $z' \in U_z$, the fiber over $z'$, $\pi^{-1}(z')$, is diffeomorphic to $\{1, \cdots, n\}$, the set $M$ is in fact a smooth manifold, with dimension  $5 + 0 =5$ (see for example Proposition 1.1.14 in~\cite{Saunders}). The neighborhood $U_z$ is chosen small enough so that there are no two identical fibers and the parametrization of the fiber remains the same for all points of $U_z$.  
\end{proof}

The practical implication of this theorem is that all the information, one can expect to extract is already contained in the essential matrix. 

\subsection{Shape Recovery}

Once the deformation is known the shape before and after deformation is easily calculated. Indeed from the first image, each point is known up to a scalar multiplication (depth). From the second image, this scalar for each point is computed linearly. The complicated part is to compute the deformation and this is our focus. 

\subsection{Critical Surface}
\label{sec::critical_surface}

Are there point configurations that do not allow the recovery of the essential matrix? 
It turns out the situation is  similar to the classical case.

Assume that the projected points before and after deformation do not constrain the essential matrix uniquely. Therefore there exists more than one solution (homogeneous) to the system: $q_i^t E p_i = 0$. One is the correct solution $E_1 = [t]_\times A$, while another solution $E_2$ would have another decomposition. Therefore if there exists another solution, the points must satisfy:
\begin{equation}
\label{eq::critical_quadric}
q_i^t E_2 p_i = P_i^t \left [ \begin{array}{c} A^t \\ t^t \end{array} \right ] E_2 [I;0] P_i = 0.
\end{equation}

Let $M = \left [ \begin{array}{c} A^t \\ t^t \end{array} \right ] E_2 [I;0]$. Then equations~\eqref{eq::critical_quadric} simply means that the points $P_i$ lie on the quadric defined by $\frac{1}{2}(M+M^t)$.

In other words, this means that the original points in space lie on a quadric, whose equation involves the affine motion that we are looking for. In this case, the recovery presents an additional layer of ambiguity. There exist several essential matrices and for each essential matrix, the corresponding affine motion is recovered up to the ambiguity described above.

\subsection{Invariant Shape}
\label{subsec::invariant}

The shape of a deforming object is by definition changing but there are  descriptions that are invariant to the transformations, we shall show when these descriptions can be recovered from a sequence of images of a deforming object. 

\subsubsection{Equations}
Let $P_0,P_1,P_2,P_3,...,P_{n-1}$ be $3d$ points in homogeneous coordinates with $1$ as the last coordinate. Here and during all section~\ref{subsec::invariant}, the points $P_0,P_1,P_2,P_3$ are assumed to define a affine basis of the three dimensional affine space $\R^3$.  

If a point $P$ satisfies $P=\alpha P_0+\beta P_1+\gamma P_2+(1-\alpha-\beta-\gamma) P_3$,  if it undergoes a 3D affine transformation, $T$, then,
$$
TP=\alpha TP_0+\beta TP_1+\gamma TP_2+(1-\alpha-\beta-\gamma) TP_3.
$$
Thus, $(\alpha,\beta, \gamma)$  is an affine invariant and  $\alpha,\beta, \gamma$ and $1-\alpha-\beta-\gamma$ are the affine invariant coordinates of $P$. In this section we aim at computing this affine invariant. The transformation itself is not recovered here. We  deal with the simultaneous recovery of the transformation and the point coordinates in section~\ref{sec::recons}. 

The real advantage of this affine invariant is that it does not require camera calibration, while full recovery of deformation and structure requires it. On the other hand, the affine invariant description only provides structure up to an unknown affine deformation.   

Let us write down the equations for the two image point sets $\{q_i\}$ and $\{q_i'\}$  of an affinely changing point set where $\begin{pmatrix}A & t \\ 0 & 1\end{pmatrix}$ is the affine transformation and $C$  the unknown camera matrix.

\begin{itemize}
\item For the first image, before the deformation, for each $i$, we have: $q_i \equiv C [\alpha_i P_0+\beta_i P_1+\gamma_i P_2+(1-\alpha_i-\beta_i-\gamma_i) P_3]$. Thus  each image point gives two equations.

\item After the deformation, in the second image, for each $i$, we have: $q_i'  \equiv C \begin{pmatrix}A & t \\ 0 & 1 \end{pmatrix} P_i \equiv C \begin{pmatrix}A & t \\ 0 & 1 \end{pmatrix} [\alpha_i P_0+\beta_i P_1+\gamma_i P_2+(1-\alpha_i-\beta_i-\gamma_i) P_3]$. Again this yields two equations per point.
\end{itemize}

Basically the system has $2\times 2n = 4n$ equations, where $n$ is the number of points. As for the unknowns, there are $4\times 3+3(n-4)+ 12 + 12 = 3n+24$ unknowns namely $P_0,P_1,P_2,P_3$, $\{\alpha_i, \beta_i,\gamma_i\}_{i \geq 4}$, the camera matrix $C$ and the affine transformation $\begin{pmatrix}A & t\\0 & 1\end{pmatrix}$. 

There are still ambiguities as, for any full rank $V$  a $4\times 4$ matrix with last row $[0,0,0,1]$,
$CP=(CV)(V^{-1}P)$, (new camera $\times$ new points) and  $C\begin{pmatrix}A & t\\0 & 1\end{pmatrix}P = (CV)(V^{-1}\begin{pmatrix}A t\\0 1\end{pmatrix}V)(V^{-1}P)$, (new camera $\times$ new affine transformation $\times$ new points). Since, in the context of this section where we are interested to compute an invariant of the shape, the only unknowns that are relevant are $\{\alpha_i, \beta_i,\gamma_i\}_{i \geq 4}$, we can assume $C=[I;0]$, removing 12 unknowns. This formally makes the computation identical to the case of a calibrated camera, while calibration is not required here.

To make things more explicit, let us introduce new variables $\la_0,\la_1,\la_2,\la_3$, such that $P_i = \la_i q_i$ for $0 \leq i \leq 3$. Then the equations can be written as follows:
{\small 
\begin{align}
	q_i \wedge [\alpha_i (\la_0 q_0) +\beta_i (\la_1 q_1) + \gamma_i (\la_2 q_2) + \delta_i (\la_3 q_3)] = 0 \label{eq::1stImage} \\
	q_j' \wedge [A;t] \begin{bmatrix}\la_j q_j \\ 1 \end{bmatrix} = 0 \label{eq::2sdImageLambda} \\
	q_i' \wedge [A;t] \begin{bmatrix} \alpha_i (\la_0 q_0) +\beta_i (\la_1 q_1) + \gamma_i (\la_2 q_2) + \delta_i (\la_3 q_3) \\ 1 \end{bmatrix} = 0 \label{eq::2sdImageHybrid}
\end{align}}
where $0 \leq j \leq 3$ and $4 \leq i \leq n-1$ and $\delta_i = 1-\alpha_i-\beta_i-\gamma_i$, $\wedge$ being the cross product. 

These equations define a real algebraic variety in $\R^{12} \times \R^{4} \times \R^{3(n-4)}$ (we  discarded $\delta_i$ in this counting). 

Since none of $\{\la_0,\la_1,\la_2,\la_3\}$ should be zero, we need to compute in the localization of the polynomial ring with respect to each $\la_i$~\cite{GreuelPfister-2002}. Here again, this is  done by adding new variables $\{\mu_0,\mu_1,\mu_2,\mu_3\}$ and the  equations:
\begin{equation}
\la_i \cdot \mu_i - 1 = 0 \label{eq::localization}
\end{equation}

We end up with a real algebraic variety embedded in $\R^{12} \times \R^{4} \times \R^4 \times \R^{3(n-4)}$. Since $\{\la_0,\la_1,\la_2,\la_3,\mu_0,\mu_1,\mu_2,\mu_3\}$ and $\{a_{ij},t_k\}_{1 \leq i,j,k \leq 3}$ are not of interest, we eliminate them from the system and get a system involving only $\{\al_i,\be_i,\ga_i\}_{4 \leq i \leq n-1}$. This is equivalent to projecting $X$ over $\R^{3(n-4)}$. Notice that we are concerned with the case $n \geq 5$.  The question now is: 
does  this define a zero-dimensional variety or in other words can the affine invariant describing each point  be computed up to a finite fold ambiguity? We address this question in the following subsection.



\subsubsection{Dimension Analysis}

\paragraph{The case of two images}

The variety $X$ defined here is isomorphic to the set $\cap_{i=1}^n (f_i \circ \pi_i^n)^{-1}(0)$ that appears in theorem~\ref{thm::essential_matrix_is_all}. Indeed both express the same constraints on the same data in a slightly different parametrization. Therefore $X$ has dimension $5$. Given this fact, can the projection of $X$ on $\R^{3(n-4)}$ be a finite variety? By analysing the fibers of the projection, we shall prove that the image of $X$ by the projection has positive dimension.  

Indeed consider a point in the projection of $X$ on $\R^{3(n-4)}$. Let us denote this point $w = (\alpha_,\beta_i,\gamma_i)_{4 \leq i \leq n-1}$. And let $\pi: \R^{12} \times \R^{4} \times \R^4 \times \R^{3(n-4)} \rightarrow \R^{3(n-4)}$ be the canonical projection. We shall determine the fiber of $\pi$ over $w$ is several steps.

\begin{lemma}
\label{lemma::1stimage_1d}
Provided $n \geq 6$ and the points are in a generic position, the equations~\eqref{eq::1stImage}, where $\lambda_0, \lambda_1, \lambda_2, \lambda_3$ are considered as unknowns, define a one-dimensional linear space. 
\end{lemma}
\begin{proof}
The geometric signification of these equations is that the point $\alpha_i (\lambda_0 q_0) + \beta_i (\lambda_1 q_1) + \gamma_i (\lambda_2 q_2) + \delta_i (\lambda_3 q_3)$ lie in the ray defined by the camera center and the pixel $q_i$. This yields two independent homogeneous linear conditions on $\lambda_0, \lambda_1, \lambda_2, \lambda_3$ for each $i$. If we stack all together these equations, we get an homogeneous linear system which rank has to be less than $4$, unless there is no non-trivial solution, which is impossible since the initial structure of the points $P_0,P_1,P_2,P_3$ is actually a non-trivial solution. On the other hand, the system has at least rank $3$, unless the $(\alpha_,\beta_i,\gamma_i)_{4 \leq i \leq n-1}$ satisfy a set of algebraic constraints (the vanishing of all $3 \times 3$ sub-determinant), which contradicts the genericity assumption. Eventually, we found that the system has exactly rank $3$ and the set of solution is a linear one-dimensional space. 
\end{proof}

\begin{lemma}
\label{lemma::fiber_over_lambda}
Provided $\lambda_0, \lambda_1, \lambda_2, \lambda_3$ are known, equations~\eqref{eq::2sdImageLambda} define a four-dimensional sub linear of the space of all affine transformations. 
\end{lemma}
\begin{proof}
Each of the equations~\eqref{eq::2sdImageLambda} means that the point $[A,t] \begin{bmatrix} \lambda_j q_j \\ 1 \end{bmatrix}$ lie in the ray $L_i$ generated by $q_i'$ and the camera center. Therefore for each $i \in \{0,1,2,3\}$, we get two independent homogeneous linear equations on $(A,t)$. 

To check that by stacking together the 8 equations obtained for $i=0,1,2,3$, consider the choice of four points $Q_0,Q_1,Q_2,Q_3$, each $Q_i$ lying on $L_i$. For each such sequence, there is a unique affine transformations that map $P_i$ to $Q_i$, since the points $P_i$  form an affine basis of $\R^3$. The transformations that satisfies equations~\ref{eq::2sdImageLambda} are precisely those obtained by this procedure. 

Therefore the set of solutions is a four-dimensional linear subspace of $\R^{12}$.
\end{proof}

\begin{lemma}
Provided $n \geq 6$ and $\lambda_0, \lambda_1, \lambda_2, \lambda_3$ are known, equations~\eqref{eq::2sdImageLambda} and~\eqref{eq::2sdImageHybrid} define a one-dimensional linear subspace of the affine group of $\R^3$.  
\end{lemma}
\begin{proof}
In lemma~\ref{lemma::fiber_over_lambda}, we  proved that equations~\eqref{eq::2sdImageLambda} define a four dimensional linear subspace of $\R^{12}$. For each point $P_i$, $i \geq 4$, equations~\eqref{eq::2sdImageHybrid} yields two independent homogeneous linear equations on $(A,t)$. Provided we have at least two such points, we get a linear homogeneous system on $(A,t)$ with at least $12$ equations. 

The rank of the system cannot be full, since there is a non-trivial solution, i.e. the actual transformation undergone by the points. The rank of the system will be at least 11, unless the points $P_i$, for $i \geq 4$ satisfy one or more algebraic relations, which contradicts here again the genericity assumption. This completes the proof.
\end{proof}

\begin{corollary}
Under the assumption that the points $\{P_i\}_{0 \leq i \leq n-1}$ are in a generic configuration and provided that $n \geq 6$, the fiber of $\pi$ over $w \in \pi(X)$ is a two-dimensional smooth manifold. Therefore $\pi(X)$ cannot be a finite set, so that two images are not enough to compute the affine invariant coordinates of the points $P_i$ for $4 \leq i \leq n-1$. 
\end{corollary}
\begin{proof}
The dimension of the fiber is a direct consequence of the lemmas proven just above. Assume that $\pi(X)$ is finite. Then it is a zero-dimensional smooth manifold. In that case, the restriction of $\pi$ to $X$ is a surjective submersion and the fibers have dimension $\dim(X) - 0 =5$, which is a contraction. 
\end{proof}

One can wonder if some prior knowledge of the world can help to get a finite set of solutions. It turns out that even the knowledge of $\{P_0,P_1,P_2,P_3\}$ cannot fully allow the computation of the affine invariant coordinates as shown in the following theorem. 

\begin{theorem}
\label{thm::affine_basis_not_enough}
If the $4$ points $\{P_0,P_1,P_2,P_3\}$ are known and $n \geq 6$, the variety of affine invariant coordinates of the other points $\{P_i\}_{4 \leq i \leq n-1}$ from two images from a single non-calibrated camera has dimension $3$, even if the scene undergoes a general affine deformation and the points are in generic position.  
\end{theorem}
\begin{proof}
Since the the $4$ points $\{P_0,P_1,P_2,P_3\}$ are known, equations~\eqref{eq::2sdImageLambda} define a four-dimensional linear subspace of $\R^{12}$. Let us write the points of this space as linear combinations $\eta_1 [A_1,t_1] + \eta_2 [A_2,t_2] + \eta_3 [A_3,t_3] + \eta_4 [A_4,t_4]$, where $[A_i,t_i]$ are linearly independent affine transformations that satisfies equations~\eqref{eq::2sdImageLambda}. 

Let us plug this representation into equations~\eqref{eq::2sdImageHybrid}, we get quadratic equations on $\eta_1,\eta_2,\eta_3,\eta_4,\alpha_i,\beta_i,\gamma_i$ for $4 \leq i \leq n-1$:

\begin{equation}
\label{eq::2sdimage_into_3rdimage}
\sum_{j=1}^4 \eta_j q_i' \wedge [A_j,t_j] \begin{bmatrix} \alpha_i (\la_0 q_0) +\beta_i (\la_1 q_1) + \gamma_i (\la_2 q_2) + \delta_i (\la_3 q_3) \\ 1 \end{bmatrix} = 0,
\end{equation}

for $4 \leq i \leq n-1$. Together with equations~\eqref{eq::1stImage}, this defines a real algebraic variety in $Y$ in $\R^4 \times \R^{3(n-4)}$, which is in fact a smooth manifold of dimension $4$, as we shall prove now. 

First observe that equations~\eqref{eq::1stImage} merely mean that for each $i \geq 4$, there exists $\lambda_i \in \R$, such that $\alpha_i (\lambda_0 q_0) + \beta_i (\lambda_1 q_1) + \gamma_i (\lambda_2 q_2) + \delta_i (\lambda_3 q_3) = \lambda_i q_i$, which is equivalent to write $\alpha_i(P_0-P_3) + \beta_i(P_1 - P_3) + \gamma_i (P_2 - P_3) = \lambda_i q_i - P_3$. Let $\Delta$ be the $3 \times 3$ matrix, which columns are $P_0 - P_3, P_1-P_3, P_2 - P_3$. Then $\Delta$ is non-singular and:
\begin{equation}
\label{eq::affine_coordinates_as_1d}
\left [ \begin{array}{c} 
\alpha_i \\
\beta_i \\
\gamma_i 
\end{array}
\right] = \lambda_i \Delta^{-1}q_i - \Delta^{-1}P_3.
\end{equation}

Plugging this expression into equations~\eqref{eq::2sdimage_into_3rdimage} yields:
$$
\sum_{j=1}^4 \eta_j q_i' \wedge [A_j,t_j] \begin{bmatrix} \lambda_i q_i \\ 1 \end{bmatrix} = 0,
$$
which can also be written $\lambda_i q_i' \wedge \left ( \left ( \Sigma_{j=1}^4 \eta_j A_j \right ) q_i \right ) + q_i' \wedge \left ( \Sigma_{j=1}^4 \eta_j t_j \right ) = 0$. Let $U_i$ be the open dense set of $\R^4$, for which when $\eta = (\eta_1,\eta_2,\eta_3,\eta_4) \in U_i$, we have: $q_i' \wedge \left ( \left ( \Sigma_{j=1}^4 \eta_j A_j \right ) q_i \right ) \neq 0$. Over $U_i$, we have: $\lambda_i = \frac{\parallel q_i' \wedge \left ( \Sigma_{j=1}^4 \eta_j t_j \right ) \parallel}{\parallel q_i' \wedge \left ( \left ( \Sigma_{j=1}^4 \eta_j A_j \right ) q_i \right ) \parallel}$. Thus $\lambda_i$ is a smooth function of $\eta$ on $U_i$ and so are $\alpha_i, \beta_i, \gamma_i$. Let $U = \cap_{i=4}^{n-1} U_i$, which is also a dense open of $\R^4$. Let $f: U \rightarrow \R^{3(n-4)}$ be the smooth function that maps $\eta$ to $(\alpha_4,\beta_4,\gamma_4, \cdots, \alpha_{n-1},\beta_{n-1},\gamma_{n-1})$. Therefore $Y$ is the graph of $f$ and is therefore a smooth embedded sub-manifold of $\R^4 \times \R^{3(n-4)}$ of dimension $4$ (see~\cite{Lee-2013}, Proposition 5.7).

Now, let us consider the projection of $Y$ on $\R^{3(n-4)}$ that we shall denote $Z$. Let $\pi: Y \rightarrow Z$ be this projection. Over each point of $(\alpha_i,\beta_i,\gamma_i)_{4 \leq i \leq n-1} \in Z$, the fiber of $\pi$ is a one-dimensional linear space, provided $n \leq 6$ and the points are in generic position. Indeed equations~\eqref{eq::2sdimage_into_3rdimage} define an homogeneous system on $\eta = [\eta_1, \cdots, \eta_4]$, which matrix is  $M_i = \left [ q_i' \wedge [A_1,t_1] P_i, q_i' \wedge [A_2,t_2] P_i, q_i' \wedge [A_3,t_3] P_i, q_i' \wedge [A_4,t_4] P_i \right ],$ where we  denote $P_i = \begin{bmatrix} \alpha_i (\la_0 q_0) +\beta_i (\la_1 q_1) + \gamma_i (\la_2 q_2) + \delta_i (\la_3 q_3) \\ 1 \end{bmatrix}$ for the sake of simplicity.  The matrix $M_i$ has rank $2$, since its four columns are in the plane perpendicular to $q'_i$ and the transformation $\{[A_i,t_i]\}_{1 \leq i \leq 4}$ are linearly independent. For $k \neq l$, the columns of the matrices $M_k$ and $M_l$ define a three dimensional space, as they span the whole space $\R^3$. Indeed since $k \neq l$, the columns of these matrices lie in two un-parallel planes in $\R^3$. Therefore a third point will not bring any further constraint on $\eta$, since the columns the matrix it defines are linear combinations of the columns of $M_k$ and $M_l$. Eventually we get that equations~\eqref{eq::2sdimage_into_3rdimage} define a one-dimensional linear space over each point $(\alpha_i,\beta_i,\gamma_i)_{4 \leq i \leq n-1} \in Z$. 

Now we shall analysis the variety $Z$. Equations~\eqref{eq::2sdimage_into_3rdimage} can be seen as a homogeneous linear system on $\eta$. More precisely, let $M$ be the $3(n-4) \times 4$ matrix. obtained by stacking together all matrices $M_i$ for $4 \leq i \leq n-1$. The variety $Z$ is a determinantal algebraic variety (see~\cite{Harris-92}) defined by the vanishing of all $4 \times 4$ minors of $M$ and by equations~\eqref{eq::1stImage}. We don't know if $Z$ is smooth, but at least let $V$ be the complement of the its possible singular locus. Then $V$ is a dense open set in $Z$ and a smooth manifold. 
Let us consider the restriction of the projection $\pi: Y \rightarrow Z$ to $V_0 = \pi^{-1}(V)$. It is a surjective smooth map which fibers are all one-dimensional linear spaces. Moreover for each point $z \in V$, there is a neighborhood $W$ such that the columns of $M$ that are independent will remain the same for all point $z' \in W$. Then over $W$, all points of the inverse image $\pi^{-1}(W)$ will be given by the same parametrization, so that $\pi^{-1}(W)$ is diffeomorphic to $W \times \R$. Hence the projection $\pi:V_0 \rightarrow V$ is a fiber bundle (in fact a vector bundle). Therefore it is a surjective submersion and $\dim(Z) = \dim(Y) - 1 = 3$. 

\end{proof}

\paragraph{Three images}

In this section, we investigate the case of three images. Two configurations are possible, either the  same deformation is repeated twice or two different deformations are performed. 

The case of two distinct deformations is quickly dealt with,  relying on theorem~\ref{theo::3IsBad}, one can prove the following result:

\begin{theorem}
\label{theo::invariant_3images}
  When the points undergo two unrelated generic affine deformations, the affine invariants $(\alpha_i,\beta_i,\gamma_i)$ cannot be computed up to a finite fold ambiguity.
\end{theorem}
\begin{proof}
If we stack  equations~(\ref{eq::1stImage}),~(\ref{eq::2sdImageLambda}),~(\ref{eq::2sdImageHybrid}) and~(\ref{eq::localization}), the equations coming from a third image generated by a generic distinct affine deformations, we get a real algebraic variety $X$ embedded in $\R^{12} \times \R^{12} \times \R^8  \times \R^{3(n-4)}$. Projecting this variety over $\R^{12} \times \R^{12}$, we get the variety defined by the essential matrices. As known from theorem~\ref{theo::3IsBad}, this variety is actually a three-dimensional smooth manifold, that we shall denote $Y$. Since once the deformations are known the structure can be uniquely computed by mere triangulation, so that the points are smooth functions of the deformations. Therefore the projection $\pi_1: X \rightarrow Y$, which is surjective and smooth by construction, has a smooth left inverse $\sigma$ and is therefore a diffeomorphism. Thus $X$ is actually a smooth submanifold of $\R^{12} \times \R^{12} \times \R^8  \times \R^{3(n-4)}$ and $\dim(X) = 3$. Consider now the projection over $\R^{3(n-4)}$: $\pi_2: X \rightarrow  \R^{3(n-4)}$. The image of $X$ is a constructible set, that is a finite union of locally closed sets in the Zariski topology. Therefore there is a dense open set in the subspace topology for the classical topology, which is a smooth manifold. Let $V$ be this open set and let us restrict the co-domain of $\pi_2$ to $V$, so that $\pi_2$ is seen as a smooth map from $V_0 = \pi_2^{-1}(V)$ to $V$. 

Then $\pi_2|_{V_0}$ is a fiber bundle with $\R^2$ as generic fiber. Indeed consider a point $z \in V$. Equations~\eqref{eq::1stImage} will define $\lambda_0,\lambda_1,\lambda_2,\lambda_3$ up to a one-dimensional ambiguity and the parametrization of the solution space will remain valid for all points $z' \in W_1$, a neighborhood of $z$. Then equations~\eqref{eq::2IsBetter} and~\eqref{eq::2sdImageHybrid} will define the two affine deformations each one modulo a one-dimensional ambiguity, as in theorem~\ref{thm::affine_basis_not_enough}. Here again the parametrization of the solution space will remain valid in a neighborhood $W_2$ of $z$. This parametrization also depend on $\lambda_0,\lambda_1,\lambda_2,\lambda_3$. All together, the fibers of $\pi_2|_{V_0}$ are two dimensional manifolds diffeomorphic to $\R^2$ and $\pi_2|_{V_0}$ is locally trivial. Therefore $\dim(V) = 1$ and the affine invariant coordinates cannot be computed up to a finite fold ambiguity. 
 \end{proof}

Let us now turn our attention to the case where the points undergo the same deformations twice. Here the unknowns are exactly the same as in the case of two images: $\{\la_0,\la_1,\la_2,\la_3\}$, $\{a_{ij},t_k\}_{1 \leq i,j,k \leq 3}$, $t = [t_1,t_2,t_3]^t$ and $\{\al_i,\be_i,\ga_i\}_{4 \leq i \leq n-1}$. The equations involved in this situation also contain those of the case of two images and in addition equations similar to~(\ref{eq::2sdImageLambda}) and~(\ref{eq::2sdImageHybrid}). Finally, with $P_i = \alpha_i (\la_0 q_0) +\beta_i (\la_1 q_1) + \gamma_i (\la_2 q_2) + \delta_i (\la_3 q_3)$, we get:

\begin{align}
	q_i \wedge P_i = 0 \\	
	q_j' \wedge [A;t] \begin{bmatrix} \la_j q_j \\ 1 \end{bmatrix} = 0 \label{eq::lambdas} \\	
	q_i' \wedge [A;t] \begin{bmatrix} P_i \\ 1 \end{bmatrix} = 0 \\
	q_j'' \wedge [A^2;At+t] \begin{bmatrix} \la_j q_j \\ 1 \end{bmatrix} = 0 \label{eq::3rdImageLambda} \\
	q_i'' \wedge [A^2;At+t] \begin{bmatrix} P_i \\ 1 \end{bmatrix} = 0 \label{eq::3sdImageHybrid}
\end{align}
As above, we add to these equations, the localization constraints expressed in~\ref{eq::localization}. All together we get a variety $X \subset \R^{12} \times \R^8 \times \R^{3(n-4)}$. Again, we are interested in the projection of these variety into the factor $\R^{3(n-4)}$. However here we are in a position to prove the following result.
\begin{theorem}
\label{thm::aff_invariant_2images}
If the points undergo the same deformation twice,  one can compute the affine invariant structure, i.e. $\{\al_i,\be_i,\ga_i\}_{4 \leq i \leq n-1}$ up to a finite fold ambiguity from the three images.
\end{theorem}
\begin{proof}
The proof is quite clear and works with the same scheme as the previous proofs. By eliminating from the equations the  variables other than the affine deformations, we get exactly the same equations of  the essential matrices. From theorem~\ref{theo::2IsBetter}, we know that there is a single solution for the affine deformation. Then the other variables are uniquely determined. As a consequence, $\dim(X) = 0$. Therefore if one first eliminates the variables related to the deformation and the $\lambda_i$,  a discrete variety for the affine invariant $(\alpha_i,\beta_i,\gamma_i)$ is left.  
\end{proof}

On the practical side, the variety which points are $\{\al_i,\be_i,\ga_i\}_{4 \leq i \leq n-1}$, that is the projection of $X$ into $\R^{3(n-4)}$ is defined by equations~\eqref{eq::1stImage} and all $4 \times 4$ minors of the matrix $M$ defined in the proof of theorem~\ref{thm::affine_basis_not_enough}. From equations~\eqref{eq::1stImage}, we can express any vector of affine coordinates as a linear function of the single parameter $\lambda_i$ as in equation~\eqref{eq::affine_coordinates_as_1d}. Then the minors of $M$ yields non-linear equations on $\{\lambda_i\}_{4 \leq i \leq n-1}$, which can be solved numerically.

\section{General Smooth Deformations}

Consider a non-singular complete vector field on $\R^3$, denoted $X$. Let $\Phi: \R^3 \times \R \rightarrow \R^3$ be the flow of $X$, i.e. $\forall x \in \R^3, \frac{\partial \Phi}{\partial t}(x,t) \mid_{t=0} = X_x$. Let $\delta t$ be a small duration and $\delta x$ a small vector, we have:
$$
\Phi(x + \delta x,t+\delta t) \approx \Phi(x,t) + \delta t \frac{\partial \Phi}{\partial t} (x,t) + \frac{\partial \Phi}{\partial x} (x,t) \delta x = \Phi(x,t) + \delta t X_{\Phi(x,t)} + \left (\frac{\partial X}{\partial x}\right)_{\Phi(x,t)} \delta x,
$$
where $\left (\frac{\partial X}{\partial x}\right)_{\Phi(x,t)}$ is the Jacobian matrix of $X$ computed at $\Phi(x,t)$. In this equation, we  used the canonical identification between a vector space and its tangent space at any point. 

If the time separation between consecutive frames is small in comparison to values of the vector field $X$ and if the distance between the points $\{P_i\}_{i=1, \ldots, n}$ is small in comparison to the spatial variability of $X$, the transformation between consecutive frames can be approximated by an affine deformation $[A,a]$, where 
$$
\left \{ \begin{array}{rcl}
    A  & = &  \left (\frac{\partial X}{\partial x}\right)_{\Phi(x,t)} \\
    a  & = & \delta t X_{\Phi(x,t)} 
\end{array} \right.
$$

In that scenario, the motion can be described by a sequence of affine deformations. If the video frequency is high enough, two consecutive deformations are quite similar and by theorems~\ref{theo::2IsBetter} and~\ref{thm::2images_is_better++}, one can recover the deformation and the structure, or by theorem~\ref{thm::aff_invariant_2images} one can recover the affine invariant coordinates. 

By this approach, one can recover a complex deformation by successive approximations.


\section{Conclusion}

We  introduced a new problem in multiple-view geometry, i.e. the recovery of structure and deformation from a single perspective  camera, where the deformation is either an affine transformation or a general smooth deformation defined as the flow of slowly varying vector field and the camera is either calibrated or not. We showed several theoretical results and in the course of the theoretical analysis provided concrete algorithms, many of them are merely linear. This paves the way for further theoretical and practical research about deformable configurations of points viewed from a monocular sequence.

\bibliographystyle{amsplain}

\end{document}